\documentclass[10pt,twocolumn,letterpaper]{article}
\usepackage{cvpr}
\usepackage{times}
\usepackage{graphicx}
\usepackage{amsmath}
\usepackage{amssymb}

\usepackage{algorithm}
\usepackage[noend]{algpseudocode}
\usepackage{xfrac} \usepackage{tabularx} \usepackage{amsthm}
\usepackage{accents} \usepackage{mathtools} \usepackage[inline]{enumitem}

\usepackage{pgfplots}
\pgfplotsset{compat=1.5, every axis/.append style={font=\small, /pgf/number format/1000 sep={}}}
\usepackage{tikz}
\usetikzlibrary{quotes, arrows.meta, angles, calc, 3d, shapes, intersections, plotmarks} \usepgfplotslibrary{statistics, polar, groupplots}
\pgfplotsset{every mark/.append style={solid},} \usepackage{tikzscale}

\makeatletter\let\captiontemp\@makecaption\makeatother
\usepackage[font=footnotesize,labelformat=simple]{subcaption}
\makeatletter\let\@makecaption\captiontemp\makeatother
 
\usepackage[pagebackref=true,breaklinks=true,letterpaper=true,colorlinks,bookmarks=false]{hyperref}

\cvprfinalcopy 
\def\cvprPaperID{6633} 

\ifcvprfinal\fi

\DeclareMathOperator*{\argmin}{arg\,min}

\DeclareMathOperator{\vMF}{vMF}

\DeclareMathOperator{\PN}{PN}
\DeclareMathOperator{\qPN}{qPN}

\newcommand{\bI}{\mathbf{I}}
\newcommand{\bK}{\mathbf{K}}
\newcommand{\bR}{\mathbf{R}}

\newcommand{\ba}{\mathbf{a}}
\newcommand{\bbb}{\mathbf{b}}
\newcommand{\bc}{\mathbf{c}}
\newcommand{\bd}{\mathbf{d}}

\newcommand{\bbf}{\mathbf{f}}
\newcommand{\bp}{\mathbf{p}}
\newcommand{\br}{\mathbf{r}}
\newcommand{\bt}{\mathbf{t}}

\newcommand{\bx}{\mathbf{x}}

\newcommand{\bO}{\mathbf{O}}
\newcommand{\sC}{\mathcal{C}}
\newcommand{\sN}{\mathcal{N}}
\newcommand{\sV}{\mathcal{V}}
\newcommand{\sS}{\mathcal{S}}

\newcommand{\sP}{\mathcal{P}}
\newcommand{\sF}{\mathcal{F}}

\newcommand{\bbR}{\mathbb{R}}
\newcommand{\bbS}{\mathbb{S}}

\newcommand{\bmu}{\boldsymbol{\mu}}
\newcommand{\btheta}{\boldsymbol{\theta}}

\newcommand{\defeq}{\triangleq}
\newcommand{\transpose}{^{\intercal}}
\DeclareRobustCommand{\quartiles}[3]{#2\hspace{1pt}\rlap{\textsubscript{#1}}{\textsuperscript{\raisebox{1pt}{#3}}}}
\DeclareRobustCommand{\overbar}[1]{\mkern 2mu\overline{\mkern-2mu#1}}
\DeclareRobustCommand{\underbar}[1]{\underline{#1\mkern-2mu}\mkern 2mu}

\makeatletter
\let\OldStatex\Statex
\renewcommand{\Statex}[1][3]{
  \setlength\@tempdima{\algorithmicindent}
  \OldStatex\hskip\dimexpr#1\@tempdima\relax}
\makeatother

\newtheorem{theorem}{Theorem}
\newtheorem{lemma}{Lemma}

\newcommand{\figref}[1]{Figure~\ref{#1}}
\newcommand{\eqnref}[1]{(\ref{#1})}

\newcommand{\tabref}[1]{Table~\ref{#1}}
\newcommand{\algoref}[1]{Algorithm~\ref{#1}}
\newcommand{\lineref}[1]{line~\ref{#1}}

\usepackage[most]{tcolorbox}

\begin{document}

\title{The Alignment of the Spheres:\\Globally-Optimal Spherical Mixture Alignment for Camera Pose Estimation}

\author{Dylan Campbell\textsuperscript{1}, Lars Petersson\textsuperscript{1,2}, Laurent Kneip\textsuperscript{3}, Hongdong Li\textsuperscript{1} and Stephen Gould\textsuperscript{1}\\
\begin{minipage}{0.31\textwidth}
\centering\textsuperscript{1}Australian National University\\
\centering{\tt\small firstname.lastname@anu.edu.au\vphantom{p}}
\end{minipage}\hspace{12pt}
\begin{minipage}{0.31\textwidth}
\centering\textsuperscript{2}Data61/CSIRO\\
\centering{\tt\small lars.petersson@data61.csiro.au}
\end{minipage}\hspace{12pt}
\begin{minipage}{0.31\textwidth}
\centering\textsuperscript{3}ShanghaiTech\\
\centering{\tt\small lkneip@shanghaitech.edu.cn}
\end{minipage}\vspace{1pt}}

\maketitle

\begin{abstract}
Determining the position and orientation of a calibrated camera from a single image with respect to a 3D model is an essential task for many applications. When 2D--3D correspondences can be obtained reliably, perspective-n-point solvers can be used to recover the camera pose. However, without the pose it is non-trivial to find cross-modality correspondences between 2D images and 3D models, particularly when the latter only contains geometric information. Consequently, the problem becomes one of estimating pose and correspondences jointly. Since outliers and local optima are so prevalent, robust objective functions and global search strategies are desirable. Hence, we cast the problem as a 2D--3D mixture model alignment task and propose the first globally-optimal solution to this formulation under the robust $L_2$ distance between mixture distributions. We search the 6D camera pose space using branch-and-bound, which requires novel bounds, to obviate the need for a pose estimate and guarantee global optimality. To accelerate convergence, we integrate local optimization, implement GPU bound computations, and provide an intuitive way to incorporate side information such as semantic labels. The algorithm is evaluated on challenging synthetic and real datasets, outperforming existing approaches and reliably converging to the global optimum.
\end{abstract}

\section{Introduction}
\label{sec:introduction}

Estimating the pose of a calibrated camera given a single image and a 3D model, is useful for many applications, including object recognition \cite{aubry2014seeing}, motion segmentation \cite{olson2001general}, augmented reality \cite{marchand2016pose}, and localization \cite{fischler1981random, kneip2015sdicp}.
The problem can be cast as a 2D--3D alignment problem in the image plane or on the unit sphere. The task is to find the rotation and translation that aligns the projection of a 3D model with the 2D image data, using points \cite{david2004softposit, campbell2018globally}, lines \cite{brown2015globally}, silhouettes \cite{cheung2003visual}, or mixture models \cite{baka2014oriented}.
This is visualized in \figref{fig:spherical_alignment} for mixture models on the unit sphere.

\begin{figure}[!t]\centering
    \def\svgwidth{\columnwidth}
    \input{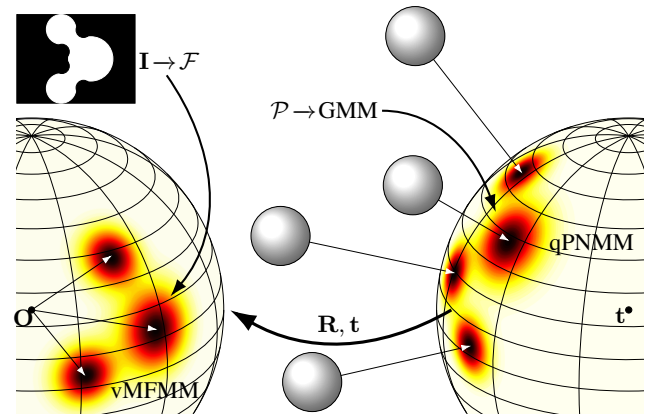}
	\caption{Spherical mixture alignment for estimating the 6-DoF absolute pose ($\bR, \bt$) of a camera from a single image $\bI$, relative to a 3D model (\eg point-set $\sP$), without 2D--3D correspondences. Our algorithm recovers the transformation by generating mixture distributions from the data --- a von Mises--Fisher Mixture Model (vMFMM) from the image via a bearing vector set $\sF$ and a Gaussian Mixture Model (GMM) from the 3D model, projected onto the sphere as a quasi-Projected Normal Mixture Model (qPNMM) --- and applying branch-and-bound with tight novel bounds to find $\bR$ and $\bt$ that optimally aligns these spherical mixtures.
	}
	\label{fig:spherical_alignment}
\end{figure}

When 2D--3D correspondences are known, this becomes the well-studied Perspective-\emph{n}-Point (P\emph{n}P) problem \cite{lepetit2009epnp,hesch2011direct}. However, correspondences between 2D and 3D modalities can be difficult to estimate, not least for the general case of aligning an image with a texture-less 3D model. Even when the model contains visual information, such as SIFT features \cite{lowe2004distinctive}, repetitive elements, occlusions, and appearance variations due to lighting and weather make the correspondence problem non-trivial.
Methods that solve for pose and correspondences jointly avoid these problems. They include local optimization approaches \cite{david2004softposit, moreno2008pose}, which can only yield correct results when a good pose prior is provided, and randomized global search \cite{fischler1981random}, which becomes computationally intractable as the problem size increases.
In contrast, globally-optimal approaches \cite{brown2015globally, campbell2018globally} obviate the need for pose priors and guarantee optimality.

This work proposes the first globally-optimal solution to the 2D--3D mixture alignment problem for camera pose estimation, depicted in \figref{fig:spherical_alignment}. The algorithm optimizes the robust $L_2$ density distance and guarantees global optimality by using the branch-and-bound framework, addressing the twin challenges of outliers and non-convexity.
It provides a geometric solution without assuming that correspondences, pose priors, or training data are available.

The primary contributions are
\begin{enumerate*}[label=(\roman*)]
    \item a new closed-form mixture distribution on the sphere, the quasi-Projected Normal mixture, that approximates the projection of a 3D Gaussian mixture;
    \item a new robust objective function, the $L_2$ distance between von Mises--Fisher and quasi-Projected Normal mixture distributions;
    \item an extension of the objective function to exploit information from deep networks to accelerate convergence;
    \item a fast local optimization algorithm using the objective function and closed-form gradient;
    \item novel bounds on the objective function; and
    \item a globally-optimal algorithm for camera pose estimation, with bound computations implemented on the GPU.
\end{enumerate*}

An advantage of this approach is that aligning densities is closer to the fundamental 2D--3D problem of aligning physical and imaged surfaces than aligning discrete point samples, since densities model the underlying surfaces with arbitrarily accurate estimates \cite{devroye1987course}, albeit at the limit.
Another advantage is that it leverages the adaptive compression properties of mixture model clustering algorithms, enabling the processing of large noisy point-sets.
In addition, the continuous objective function admits the use of local gradient-based optimization, which greatly expedites convergence.
The algorithm can also be applied to a wide range of 3D data, including mesh and volume representations as well as point-sets.
Finally, the approach solves the problem of extracting geometrically-meaningful elements in 2D and 3D by (optionally) using semantic information during optimization.
This simple but effective extension reduces runtime and susceptibility to degenerate poses, using only easily-obtainable information.

\section{Related Work}
\label{sec:related_work}

When 2D--3D correspondences are known, P\emph{n}P solvers \cite{lepetit2009epnp, hesch2011direct} can accurately estimate the camera pose.
However, outliers are almost always present in the correspondence set. When this is the case, the inlier set can be retrieved using RANSAC \cite{fischler1981random} or robust global optimization \cite{enqvist2008robust, ask2013optimal, enqvist2015tractable, svarm2016city}.
Some of these approaches \cite{fischler1981random, enqvist2008robust} can be applied when correspondences are not available by providing all possible permutations of the correspondence set. However this hard combinatorial problem quickly becomes infeasible.
Matching and filtering techniques have also been developed for large-scale localization problems to reduce the number of outliers in the initial set \cite{sattler2011fast, li2012worldwide, zeisl2015camera, enqvist2015tractable, svarm2016city, sattler2017efficient}.
These methods are only practical when 2D--3D correspondences can be found and so are mostly used with Structure-from-Motion (SfM) point-sets. Each 3D point in these datasets is at a visually-distinctive location and is augmented with an image feature, simplifying the correspondence problem. This is not the case for standard point-sets, which contain only geometric information.

The problem is more complex when correspondences are not available at the outset.
Local optimization approaches include SoftPOSIT \cite{david2004softposit}, which iterates between solving for correspondences and solving for pose, and 2D/3D GMM registration \cite{baka2014oriented}, which projects 3D points into the camera plane then applies 2D Gaussian mixture alignment. This formulation treats points close to the camera identically to distant points and so neglects 3D scale information and creates false optima. Moreover, these methods only find locally-optimal solutions within the convergence basin of the provided pose prior.
To alleviate this, global optimization approaches have been proposed, including random-start local search \cite{david2004softposit} and BlindPnP \cite{moreno2008pose}, which uses Kalman filtering to search over a probabilistic pose prior.
RANSAC and variants \cite{grimson1990object} do not require a pose prior, but are only tractable for small numbers of points and outliers.
Other approaches use regression forests or convolutional neural networks to learn 2D--3D correspondences from the data and thereby regress pose \cite{shotton2013scene, kendall2015posenet, brachmann2017dsac, kendall2017geometric}.
These methods require a large training set of pose-labeled images, do not localize the camera with respect to an explicit 3D model, and cannot guarantee optimality.

Globally-optimal approaches can provide this guarantee without needing a pose estimate. They certify that the computed camera pose is a global optimizer of the objective function. The Branch-and-Bound (BB) \cite{land1960automatic} algorithm has been widely used for this purpose, with tractability continuing to be a significant impediment. For example, BB has been used for 2D--2D registration \cite{breuel2003implementation}, relative pose estimation \cite{hartley2009global}, 3D--3D rotational registration \cite{li20073d}, 3D--3D registration with known correspondences \cite{olsson2009branch}, full 3D--3D registration \cite{yang2016goicp}, and robust 3D--3D registration \cite{campbell2016gogma}.

For 2D--3D registration, Brown~\etal \cite{brown2015globally} proposed a globally-optimal method using BB with a geometric error. Trimming was used to make the objective function robust to outliers. However this requires knowing the true outlier fraction in advance; if incorrectly specified, the optimum may not occur at the correct pose.
Campbell \etal \cite{campbell2017globally, campbell2018globally} proposed a globally-optimal inlier set cardinality maximization solution to the problem. While robust, this objective function is discrete and challenging to optimize, and operates on sampled points instead of the underlying surfaces.

Our work is the first globally-optimal $L_2$ density distance minimization solution to the camera pose estimation problem. It removes the assumptions that correspondences, training data or pose priors are available and is guaranteed to find the optimum of a robust objective function.

\section{Probability Distributions on the Sphere}
\label{sec:distributions}

2D directional data such as bearing vectors can be represented as points on the unit 2-sphere. These can be treated as samples from an underlying probability distribution in $\bbS^2$. For images, this distribution models the projection of visible surfaces onto the sphere.
In this section, we will outline the probability distributions used in this work and derive a closed-form approximation for the last.
The distributions referred to in this paper are summarized in \tabref{tab:distributions_acronyms}.

\begin{table}[!t]\centering
	\caption{Probability distributions in $\bbR^3$ and $\bbS^2$.}
	\label{tab:distributions_acronyms}
	\newcolumntype{C}{>{\centering\arraybackslash}X}
	\renewcommand*{\arraystretch}{1.3}
	\setlength{\tabcolsep}{2pt}
	\begin{tabularx}{\columnwidth}{@{}l C C C@{}}\hline
		Distribution & Notation & Parameters & Manifold\\\hline
		Gaussian & $\sN$ & $\bmu, \sigma^2$ & $\bbR^3$\\
		Projected Normal & $\PN$ & $\bmu, \sigma^2$ & $\bbR^3$\\
		quasi-Projected Normal & $\qPN$ & $\bmu, \sigma^2$ & $\bbS^2$\\
		von Mises--Fisher & $\vMF$ & \hspace{-3pt}$\hat{\bmu}, \kappa$ & $\bbS^2$\\\hline
	\end{tabularx}
\end{table}

The von Mises--Fisher distribution (vMF) \cite{fisher1953dispersion}, visualized in Figure~\ref{fig:distributions_vmf}, is the spherical analog of the isotropic Gaussian distribution and has a closed form in 3D, unlike more expressive non-isotropic distributions \cite{kent1982fisher}. The probability density function of the vMF distribution in 3D is
\begin{equation}
\label{eqn:vMF}
\vMF\left(\bbf \mid \hat{\bmu}, \kappa\right) = \frac{\exp (\kappa \hat{\bmu}\transpose \bbf)}{2\pi Z(\kappa)}
\end{equation}
for the random unit bearing vector $\bbf$, mean direction $\hat{\bmu}$, and concentration $\kappa > 0$, and where
\begin{equation}
\label{eqn:Z}
Z(\kappa) = \left(\exp\left(\kappa\right) - \exp\left(-\kappa\right)\right)\kappa^{-1}\text{.}
\end{equation}

The Projected Normal (PN) distribution \cite{mardia1972statistics, watson1983statistics, wang2013directional} is the projection of a Gaussian distribution $\sN$ onto the sphere. That is, if a random variable $\bp$ follows a Gaussian distribution, then the bearing vector $\bbf = \bp/\| \bp \|$ follows a PN distribution. For a Gaussian mixture that models the distribution of 3D surfaces in a scene, the associated PN mixture models the scene as observed by a 2D sensor, albeit without visibility constraints.
The probability density function of the isotropic PN distribution in 3D \cite{pukkila1988pattern} is
\begin{equation}
\label{eqn:PN}
\PN(\bbf \mid \bmu, \sigma^2) = \frac{e^{\frac{-\rho^2}{2}}}{2\pi} \! \left[  \frac{\alpha}{\sqrt{2\pi}} \!+\! \Phi\left(\alpha\right) e^{\frac{\alpha^2}{2}} \left( 1 \!+\! \alpha^2 \right) \right]
\end{equation}
for the bearing vector $\bbf$, mean position $\bmu \in \bbR^3$, and variance $\sigma^2$, and where $\rho = \|\bmu\| / \sigma$, $\alpha = \rho \bmu\transpose\bbf / \|\bmu\|$, and $\Phi(\cdot)$ is the cumulative distribution function of $\sN$.

While PN is the true distribution, it does not have a closed form. Moreover, similarity measures between PN distributions, such as the $L_2$ distance, are not tractable to compute, since they do not simplify to a closed form when integrated over the sphere and would therefore require time-consuming numerical integration. As a result, it is impractical for alignment problems. Instead, we propose a new closed-form distribution, the \emph{quasi-Projected Normal} (qPN) distribution, that approximates a PN with a vMF distribution. Its probability density function is given by
\begin{equation}
\label{eqn:qPN}
\qPN\left(\bbf \mid \bmu, \sigma^2 \right) = \vMF\left(\bbf \;\middle|\; \frac{\bmu}{\|\bmu\|}, \left(\frac{\|\bmu\|}{\sigma}\right)^2 \!+\! 1\right)\text{.}
\end{equation}
This was derived by equating the vMF and PN density functions at $\bbf = \hat{\bmu} = \bmu / \|\bmu\|$, since they should evaluate to the same value in the direction of the mean vector.
This gives
\begin{equation}
\label{eqn:vMF_PN_equated}
\frac{\kappa}{2\pi\left(1 - e^{-2\kappa}\right)} = \frac{e^{\frac{-\rho^2}{2}}}{2\pi} \!\left[ \frac{\rho}{\sqrt{2\pi}} + \Phi\left(\rho\right) e^{\frac{\rho^{2}}{2}} \left( 1 \!+\! \rho^{2} \right) \right]
\end{equation}
which simplifies as $\kappa \to \infty$ and $\rho = \|\bmu\| / \sigma \to \infty$ to
\begin{equation}
\label{eqn:vMF_PN_map}
\kappa = \left(\frac{\|\bmu\|}{\sigma}\right)^{\!\!2} + 1\text{.}
\end{equation}
While this derivation only proves equality in the limit in the direction of the mean vector, the empirical results in Figure~\ref{fig:qpn_pdf} show that the distributions are very similar across the entire angular range, even for low values of $\rho$.

\begin{figure}[!t]
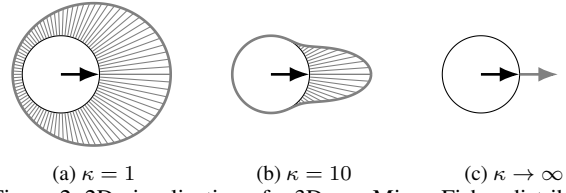
\centering
	\begin{subfigure}[]{0.33\columnwidth}\centering
		\input{vmf1.tex}
		\caption{$\kappa = 1$}
		\label{fig:distributions_vmf1}
	\end{subfigure}\hfill
	\begin{subfigure}[]{0.33\columnwidth}\centering
		\input{vmf2.tex}
		\caption{$\kappa = 10$}
		\label{fig:distributions_vmf2}
	\end{subfigure}\hfill
	\begin{subfigure}[]{0.33\columnwidth}\centering
\begin{tikzpicture}
\clip(1.3,1.0) rectangle (3.5,2.95);
\begin{polaraxis}[
width=2\textwidth,
hide x axis,
hide y axis,
]

\addplot[color=white, mark=none, data cs=cart] coordinates {(-0.13, -0.19) (0.29, -0.19) (0.29, 0.19) (-0.13, 0.19) (-0.13, -0.19)};

\addplot[color=black, mark=none, domain=0:360, samples=361] {0.1};
\addplot [polar comb, color=black, line width=1pt, -Latex] coordinates {(0,0.1)};
\addplot [color=gray, line width=1pt, -Latex, data cs=cart] coordinates {(0.1,0) (0.2,0)};

\end{polaraxis}
\end{tikzpicture}%
		\caption{$\kappa \to \infty$}
		\label{fig:distributions_vmf4}
	\end{subfigure}
	\caption{2D visualization of a 3D von Mises--Fisher distribution as the concentration parameter $\kappa$ increases. As $\kappa \to \infty$, the distribution approaches a delta function on the sphere.}
	\label{fig:distributions_vmf}
\end{figure}

\begin{figure}[!t]\centering
	\begin{subfigure}[]{0.5\columnwidth}\centering
		\input{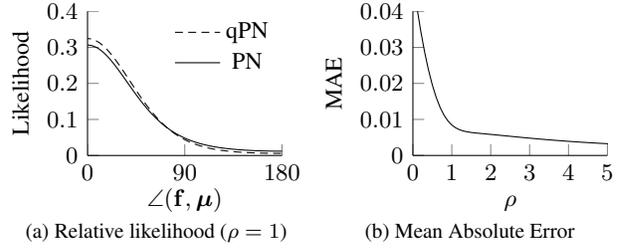}
		\vspace{-18pt}
		\caption{Relative likelihood ($\rho = 1$)}
		\label{fig:qpn_pdf_comparison}
	\end{subfigure}\hfill
	\begin{subfigure}[]{0.5\columnwidth}\centering
		\begin{tikzpicture}
\begin{axis}[
scaled ticks=false,
tick label style={/pgf/number format/fixed},
width=\textwidth,
height=3.5cm,
xmin=0,
xmax=5,
xtick={0, 1, ..., 5},
xlabel={$\rho \vphantom{\angle(\bbf, \bmu)}$},
xlabel shift=-4pt,
ymin=0,
ymax=0.04,
ytick={0, 0.01, ...,0.04},
ylabel={MAE},
axis x line*=bottom,
axis y line*=left,
]


\addplot [black] table [row sep=newline]{
0.100000000000000	0.0413191281236007
0.200000000000000	0.0346139984696496
0.300000000000000	0.0287655282865310
0.400000000000000	0.0237540068813730
0.500000000000000	0.0195422603081568
0.600000000000000	0.0160829910864740
0.700000000000000	0.0133195988920811
0.800000000000000	0.0111819311945709
0.900000000000000	0.00959030244918633
1	0.00845004251628531
1.10000000000000	0.00766221054548387
1.20000000000000	0.00713263091784733
1.30000000000000	0.00678074719239092
1.40000000000000	0.00654514421515479
1.50000000000000	0.00637890241051719
1.60000000000000	0.00625222568688656
1.70000000000000	0.00614385236628324
1.80000000000000	0.00604225040266751
1.90000000000000	0.00594325703404724
2	0.00584243421372227
2.10000000000000	0.00573946498348980
2.20000000000000	0.00563494957933974
2.30000000000000	0.00552965882016238
2.40000000000000	0.00542408740054571
2.50000000000000	0.00531819228966805
2.60000000000000	0.00521173034184175
2.70000000000000	0.00510893640257420
2.80000000000000	0.00500440880697224
2.90000000000000	0.00490423025006244
3.00000000000000	0.00480464961892366
3.10000000000000	0.00470596094922317
3.20000000000000	0.00460899176414898
3.30000000000000	0.00451506785319796
3.40000000000000	0.00442358842394028
3.50000000000000	0.00433382063316873
3.60000000000000	0.00424710238945909
3.70000000000000	0.00416346872314328
3.80000000000000	0.00408199500296547
3.90000000000000	0.00400204393072496
4	0.00392277575732362
4.10000000000000	0.00384711043735623
4.20000000000000	0.00377563566639987
4.30000000000000	0.00370306328274511
4.40000000000000	0.00363438691018970
4.50000000000000	0.00356863657853114
4.60000000000000	0.00350096684493657
4.70000000000000	0.00344123882469621
4.80000000000000	0.00337624008011355
4.90000000000000	0.00332126663698587
5.00000000000000	0.00326021275980349
5.10000000000000	0.00320831280363775
5.20000000000000	0.00315218952938847
5.30000000000000	0.00310128823298805
5.40000000000000	0.00305087424444898
5.50000000000000	0.00299849472816950
5.60000000000000	0.00295442142273135
5.70000000000000	0.00290599225929935
5.80000000000000	0.00286044892770108
5.90000000000000	0.00281913064321619
6	0.00277412406175830
6.10000000000000	0.00273221253592677
6.20000000000000	0.00269484671225199
6.30000000000000	0.00265427356087413
6.40000000000000	0.00261227041931178
6.50000000000000	0.00257972468941359
6.60000000000000	0.00254431821113954
6.70000000000000	0.00250632851637804
6.80000000000000	0.00247084860356331
6.90000000000000	0.00244116499294787
7	0.00240910028589179
7.10000000000000	0.00237487578991004
7.20000000000000	0.00234078460131801
7.30000000000000	0.00231503276864338
7.40000000000000	0.00228723026026129
7.50000000000000	0.00225755230706364
7.60000000000000	0.00222617117499046
7.70000000000000	0.00219797456116073
7.80000000000000	0.00217507443548226
7.90000000000000	0.00215049223483861
8	0.00212436344258786
8.10000000000000	0.00209682136295644
8.20000000000000	0.00206799676653635
8.30000000000000	0.00204845650327287
8.40000000000000	0.00202794580224906
8.50000000000000	0.00200610148667098
8.60000000000000	0.00198302589451114
8.70000000000000	0.00195881974468721
8.80000000000000	0.00193358190998049
8.90000000000000	0.00191431249473836
9	0.00189725321466546
9.10000000000000	0.00187908073990816
9.20000000000000	0.00185987160506170
9.30000000000000	0.00183970132685681
9.40000000000000	0.00181864425438302
9.50000000000000	0.00179677342856303
9.60000000000000	0.00177741929741420
9.70000000000000	0.00176350894505390
9.80000000000000	0.00174868562616075
9.90000000000000	0.00173300508406668
10	0.00171652245632889
};

\end{axis}
\end{tikzpicture}
		\vspace{-18pt}
		\caption{Mean Absolute Error}
		\label{fig:qpn_pdf_mae}
	\end{subfigure}
	\caption{Comparison of the qPN and PN distributions.
		\subref{fig:qpn_pdf_comparison} The qPN and PN probability density functions are plotted with respect to the angle $\angle(\bbf, \bmu)$ for $\rho = \|\bmu\| / \sigma = 1$.
		The distributions are very similar, even for this small value of $\rho$. \subref{fig:qpn_pdf_mae} The Mean Absolute Error (MAE) across the entire angular range is plotted with respect to $\rho$ and is less than $0.01$ for all $\rho \geqslant 1$.}
	\label{fig:qpn_pdf}
\end{figure}

\section{Spherical Mixture Alignment}
\label{sec:alignment}

The alignment of mixture distributions to estimate relative sensor pose
is a well-studied problem in $\bbR^2$, $\bbR^3$ \cite{chui2000feature, tsin2004correlation, jian2011robust, campbell2016gogma}, and the sphere $\bbS^2$ \cite{straub2017efficient}.
For 2D--3D camera pose estimation, we require a 3D positional and a 2D directional mixture distribution to model the input data. We model the distribution of 3D points $\bp \in \bbR^3$ in the set $\sP = \{\bp_i\}_{i=1}^{N_1}$ as a Gaussian Mixture Model (GMM). Let $\btheta_{1} = \left\{ \bmu_{1i}, \sigma_{1i}^2, \phi_{1i} \right\}_{i = 1}^{n_1}$ be the parameter set of the $n_1$-component GMM with means $\bmu_{1i} \in \bbR^3$, variances $\sigma_{1i}^2$, and mixture weights $\phi_{1i} \geqslant 0$, where $\sum \phi_{1i} = 1$, with density
\begin{equation}
\label{eqn:GMM}
p\left(\bp \mid \btheta_1\right) = \sum_{i = 1}^{n_1} \phi_{1i} \sN \left(\bp \mid \bmu_{1i}, \sigma_{1i}^2 \right)\text{.}
\end{equation}
We also require a projection of this distribution onto the sphere. For this, we use the qPN Mixture Model (qPNMM) associated with this GMM, with density
\begin{equation}
\label{eqn:qPNMM}
p\left(\bbf \mid \btheta_1\right) = \sum_{i = 1}^{n_1} \phi_{1i} \qPN \left(\bbf \mid \bmu_{1i}, \sigma_{1i}^2 \right)\text{.}
\end{equation}
Finally, we model the distribution of bearing vectors $\bbf \in \bbS^2$ in the set $\sF = \{\bbf_i\}_{i=1}^{N_2}$ as a vMF Mixture Model (vMFMM) \cite{gopal2014mises,straub2015small}.
Let $\btheta_{2} = \{\hat{\bmu}_{2j} , \kappa_{2j}, \phi_{2j} \}_{j = 1}^{n_2}$ be the parameter set of the $n_2$-component vMFMM with mean directions $\hat{\bmu}_{2j} \in \bbS^{2}$, concentrations $\kappa_{2j} > 0$, and mixture weights $\phi_{2j} \geqslant 0$, where $\sum \phi_{2j} = 1$, with density
\begin{equation}
\label{eqn:vMFMM}
p\left(\bbf \mid \btheta_2\right) = \sum_{j = 1}^{n_2} \phi_{2j} \vMF\left(\bbf \mid \hat{\bmu}_{2j}, \kappa_{2j} \right)\text{.}
\end{equation}
The bearing vector $\bbf \propto \bK^{-1}\hat{\bx}$ corresponds to a 2D point with homogeneous coordinates $\hat{\bx}$ imaged by a calibrated camera, with intrinsic camera matrix $\bK$.
The above mixture distributions admit arbitrarily accurate estimates of noisy surface densities \cite{devroye1987course} and can be computed efficiently from the data \cite{campbell2015adaptive, kulis2012revisiting, straub2015small}.

The $L_2$ distance between probability densities is a robust objective function that can be used to measure the alignment of two sets of sensor data, given a specific transformation \cite{jian2011robust, straub2017efficient}. Unlike the Kullback--Leibler divergence, it is inherently robust to outliers \cite{scott2001parametric} and operates on statistical densities generated from the raw sensor data. The densities model the underlying surfaces of the scene, which is beneficial because the fundamental 2D--3D registration problem is a surface alignment problem, not a discrete sampled point alignment problem.

\begin{lemma}
\label{lm:objective_function}
($L_2$ objective function) The $L_2$ distance between qPNMM and vMFMM models with rotation $\bR \in SO(3)$ and translation $\bt \in \bbR^3$ can be minimized using the function
\begin{align}
\label{eqn:objective_function}
f(\bR,\bt) = &\sum_{i = 1}^{n_1} \sum_{j = 1}^{n_1} \frac{\phi_{1i} \phi_{1j}Z\left(K_{1i1j}(\bt)\right)}{Z(\kappa_{1i}(\bt)) Z(\kappa_{1j}(\bt))} \nonumber\\
&- 2\sum_{i = 1}^{n_1} \sum_{j = 1}^{n_2} \frac{\phi_{1i} \phi_{2j}Z\left(K_{1i2j}(\bR,\bt)\right)}{Z(\kappa_{1i}(\bt)) Z(\kappa_{2j})}
\end{align}
where
\begin{equation}
\label{eqn:K1i1j}
K_{1i1j}(\bt) = \left\|\kappa_{1i}(\bt) \frac{\bmu_{1i}-\bt}{ \|\bmu_{1i}-\bt\|} \!+\! \kappa_{1j}(\bt) \frac{\bmu_{1j}-\bt}{\|\bmu_{1j}-\bt\|}\right\|
\end{equation}
\begin{equation}
\label{eqn:K1i2j}
K_{1i2j}(\bR,\bt) = \left\|\kappa_{1i}(\bt) \bR \frac{\bmu_{1i}-\bt}{ \|\bmu_{1i}-\bt\|} + \kappa_{2j} \hat{\bmu}_{2j}\right\|
\end{equation}
\begin{equation}
\label{eqn:kappa1}
\kappa_{1i}(\bt) = \left(\frac{\|\bmu_{1i} - \bt\|}{\sigma_{1i}}\right)^{\!\!2} + 1
\end{equation}
and $Z(\cdot)$ is defined as given in \eqnref{eqn:Z}.
\end{lemma}
\begin{proof}
Given qPNMM and vMFMM models of the input data and a rigid transformation function $T(\btheta_1, \bR, \bt) = \left\{ \bR(\bmu_{1i} - \bt), \sigma_{1i}^2, \phi_{1i} \right\}_{i = 1}^{n_1}$, the $L_2$ distance between densities for a rotation $\bR$ and translation $\bt$ is given by
\begin{align}
d_{L_2} &= \int_{\bbS^{2}} \left[ p\left( \bbf \mid T(\btheta_1, \bR, \bt) \right) - p\left( \bbf \mid \btheta_2 \right) \right]^{2}\,\mathrm{d}\bbf \label{eqn:l2_distance_0}\\
&= \int_{\bbS^{2}} \left[ p\left( \bbf \mid T(\btheta_1, \bR, \bt) \right)\right]^2  + \left[ p\left( \bbf \mid \btheta_2 \right) \right]^{2} \nonumber\\
&\hspace{15pt} - 2p\left( \bbf \mid T(\btheta_1, \bR, \bt) \right)p\left( \bbf \mid \btheta_2 \right) \,\mathrm{d}\bbf \text{.} \label{eqn:l2_distance}
\end{align}
The function \eqnref{eqn:objective_function} is obtained by removing constant summands and factors, substituting \eqnref{eqn:qPNMM}, \eqnref{eqn:vMFMM}, \eqnref{eqn:qPN} and \eqnref{eqn:vMF} into \eqnref{eqn:l2_distance}, and replacing integrals of the form $\int_{\bbS^2} \exp(\bx\transpose\bbf) \,\mathrm{d}\bbf$ with the normalization constant of a vMF density with $\kappa \!=\! \|\bx\|$ and $\hat{\bmu} \!=\! \bx / \kappa$; see appendix for details.
\end{proof}

The objective is then to find a rotation and translation that minimizes the $L_2$ distance between the densities
\begin{equation}
\label{eqn:optimization}
(\bR^{\ast},\bt^{\ast}) = \argmin_{\bR,\,\bt} f(\bR, \bt)\text{.}
\end{equation}

Furthermore, if semantic class labels are available, for example using CNN--based semantic segmentation
for 2D images \cite{ronneberger2015unet, chen2018encoder, rota2018inplace} and 3D point-sets 
\cite{landrieu2018large, huang2018recurrent, tatarchenko2018tangent},
the optimization problem can be formulated as a joint $L_2$ distance minimization over the semantic classes, providing semantic-aware alignment and accelerating convergence. That is, given a class label set $\Lambda$, one can construct separate mixture distributions for each class and solve
\begin{equation}
\label{eqn:optimization_semantic}
(\bR^{\ast},\bt^{\ast}) = \argmin_{\bR,\,\bt} \sum_{l\in\Lambda} \phi_{l} f_{l} (\bR, \bt)\text{.}
\end{equation}
where $\phi_{l} \geqslant 0$ are the class weights and $f_{l}$ is the per-class function value computed according to \eqnref{eqn:objective_function}.

\section{Branch-and-Bound}
\label{sec:bb}

To solve the non-convex $L_2$ distance problem \eqnref{eqn:optimization}, the Branch-and-Bound (BB) algorithm \cite{land1960automatic} may be applied. It requires an efficient way to branch the function domain and bound the function optimum for each branch, such that the bounds converge as the branch size tends to zero. The efficiency of the algorithm depends on both the computational complexity of the bounds and how tight they are, with tighter bounds allowing suboptimal branches to be pruned quicker which reduces the search space faster.

\subsection{Branching the Domain of Camera Poses}
\label{sec:bb_branching}

To find a pose that is globally-optimal, the $L_2$ objective function must be optimized over the entire domain of 3D camera poses, the group $SE(3) = SO(3) \times \bbR^3$. For BB, the domain must be bounded, so we restrict the space of translations to the set $\Omega_t$, under the assumption that the camera is a finite distance from the 3D model. The domains are shown in \figref{fig:bb_branching_domain}.
We minimally parameterize rotation space $SO(3)$ with angle-axis vectors $\br$ with rotation axis $\hat{\br} = \br/\|\br\|$ and angle $\|\br\|$. As a result, a solid ball with radius $\pi$ in $\bbR^3$ can represent the space of all 3D rotations.
To simplify branching, we take as the rotation domain $\Omega_r$ the 3D cube that circumscribes the $\pi$-ball \cite{li20073d}.
We denote the rotation matrix obtained from $\br$ using Rodrigues' rotation formula as $\bR_{\br}$.
The space of translations $\bbR^3$ is parameterized using 3-vectors within the cuboid $\Omega_t^{\prime}$, a bounded domain.
To avoid the non-physical case where the camera is located within a small value $\zeta$ of a 3D surface, the translation domain is restricted such that $\Omega_t = \Omega_t^{\prime} \cap \{\bt \in \bbR^3 \mid \|\bmu - \bt\| \geqslant \zeta, \forall \bmu \in \btheta_1\}$.
Finally, we branch the domain into 6D cuboids (6-orthotopes) $\sC_r \times \sC_t$ using an adaptive branching strategy that chooses to subdivide the rotation or translation dimensions based on which has the greater angular uncertainty, reducing redundant branching.
    
\begin{figure}[!t]\centering\hfill
    \begin{subfigure}[]{0.4\columnwidth}\centering
		\includegraphics[width=\columnwidth]{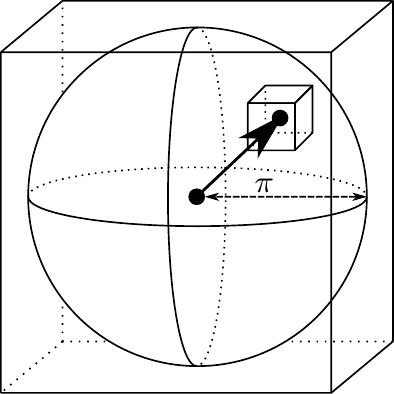}
		\caption{Rotation Domain $\Omega_r$}
		\label{fig:bb_branching_domain_rotation}
	\end{subfigure}\hfill
	\begin{subfigure}[]{0.4\columnwidth}\centering
		\includegraphics[width=\columnwidth]{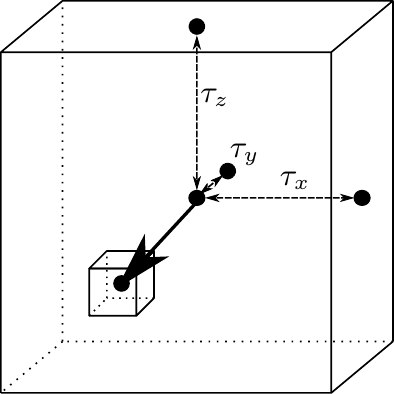}
		\caption{Translation Domain $\Omega_t$}
		\label{fig:bb_branching_domain_translation}
	\end{subfigure}\hfill\null
    \caption{Parameterization and branching of $SE(3)$. \protect\subref{fig:bb_branching_domain_rotation}~Rotations are parameterized by angle-axis vectors in a $\pi$ radius ball. \protect\subref{fig:bb_branching_domain_translation}~Translations are parameterized by vectors within a cuboid that has half-widths $[\tau_x, \tau_y, \tau_z]$. The joint domain is branched into 6D cuboids using an adaptive octree-like branching strategy.}
\label{fig:bb_branching_domain}
\end{figure}

\subsection{Bounding the $\boldsymbol{L_2}$ Objective Function}
\label{sec:bb_bounding}

The quality of the bounds is key to a successful branch-and-bound algorithm.
For $L_2$ distance minimization, we require bounds on the minimum of the objective function \eqnref{eqn:objective_function} within a transformation domain $\sC_r\times \sC_t$. An upper bound can be found by evaluating the function at any transformation in the branch.
A lower bound can be found using the bounds $\psi_r$ and $\psi_t$ on the rotation and translation uncertainty angles derived in Lemmas~3 and 5 in Campbell \etal \cite{campbell2018globally}, reproduced here as Lemmas~\ref{lm:rotation_uncertainty_angle} and \ref{lm:translation_uncertainty_angle}.
\begin{lemma}
\label{lm:rotation_uncertainty_angle}
(Rotation uncertainty angle bound) Given a 3D point $\bp$ and a rotation cube $\sC_r$ centered at $\br_0$ with surface $\sS_r$, then $\forall \br \in \sC_r$,
\begin{align}
\angle(\bR_\br \bp, \bR_{\br_0} \bp) &\leqslant \min\left\{\max_{\br \in \sS_{r}}\angle(\bR_{\br} \bp, \bR_{\br_0} \bp),\pi\right\}\nonumber\\
&\defeq \psi_{r}(\bp, \sC_r).\label{eqn:rotation_uncertainty_angle}
\end{align}
\end{lemma}
\begin{lemma}
\label{lm:translation_uncertainty_angle}
(Translation uncertainty angle bound) Given a 3D point $\bp$ and a translation cuboid $\sC_t$ centered at $\bt_0$ with vertices $\sV_t$, then $\forall \bt \in \sC_t$,
\begin{align}
\angle(\bp - \bt, \bp - \bt_0) &\leqslant
\begin{dcases}
\max_{\bt \in \sV_t}\angle(\bp - \bt, \bp - \bt_0) & \text{if } \bp \notin \sC_t\\
\pi & \text{else}
\end{dcases}\nonumber\\
&\defeq \psi_{t}(\bp, \sC_t).\label{eqn:translation_uncertainty_angle}
\end{align}
\end{lemma}

\begin{theorem}
	\label{thm:bounds}
	(Objective function bounds) For the transformation domain $\sC_r \times \sC_t$ centered at $(\br_0, \bt_0)$, the minimum of the objective function \eqnref{eqn:objective_function} has an upper bound
	\vspace{-6pt}
	\begin{equation}
	\label{eqn:bounds_upper}
	\overbar{d} \defeq f (\bR_{\br_0}, \bt_0)
	\end{equation}
	and a lower bound
	\begin{align}
    \label{eqn:bounds_lower}
    \underbar{d} \defeq &\sum_{i = 1}^{n_1} \sum_{j = 1}^{n_1} \phi_{1i} \phi_{1j} \min_{\bt \in \sC_t} \frac{Z\left(\underbar{K}_{1i1j} (\bt)\right)}{Z(\kappa_{1i} (\bt)) Z(\kappa_{1j} (\bt))} \nonumber\\
    &- 2\sum_{i = 1}^{n_1} \sum_{j = 1}^{n_2} \phi_{1i} \phi_{2j} \max_{\bt \in \sC_t} \frac{Z\left(\overbar{K}_{1i2j} (\bt)\right)}{Z(\kappa_{1i}(\bt)) Z(\kappa_{2j})}
    \end{align}
    where
    \begin{equation}
    \label{eqn:K1i1j_lb}
    \underbar{K}_{1i1j} (\bt) = \sqrt{\kappa_{1i}^{2} (\bt) \!+\! \kappa_{1j}^{2} (\bt) \!+\! 2\kappa_{1i}(\bt) \kappa_{1j}(\bt) \cos A}
    \end{equation}
    \begin{equation}
    \label{eqn:K1i2j_ub}
    \overbar{K}_{1i2j}(\bt) = \sqrt{\kappa_{1i}^{2} (\bt) + \kappa_{2j}^{2} + 2\kappa_{1i}(\bt) \kappa_{2j} \cos B}
    \end{equation}
    \begin{align}
    \label{eqn:K1i1j_lb_A}
    A = \min \{ \pi,\, &\angle\left( \bmu_{1i}-\bt_0, \bmu_{1j}-\bt_0 \right) \nonumber\\
    &+ \psi_{t}(\bmu_{1i}, \sC_t) + \psi_{t}(\bmu_{1j}, \sC_t) \}
    \end{align}
    \begin{align}
    \label{eqn:K1i2j_ub_B}
    B = \max \{ 0,\, &\angle\left(\bmu_{1i}-\bt_0, \bR_{\br_{0}}^{-1} \hat{\bmu}_{2j} \right) \nonumber\\
    &- \psi_{t}(\bmu_{1i}, \sC_t) - \psi_{r}(\hat{\bmu}_{2j}, \sC_r) \} \text{.}
    \end{align}
\end{theorem}
\begin{proof}
	The validity of the upper bound follows from
	\vspace{-6pt}
	\begin{equation}
	\label{eqn:bounds_upper_proof}
	f (\bR_{\br_0}, \bt_0) \geqslant \min_{\substack{\br \in \sC_r\\\bt \in \sC_t}} f(\bR_{\br}, \bt).
	\end{equation}
	That is, at a specific point in the domain, the function value is greater than (or equal) to the minimum in the domain.
		For the lower bound, observe that $\forall \bt \in \sC_t$,
	\begin{align}
	K_{1i1j}(\bt) &= \sqrt{
	    \begin{aligned}
	    &2\kappa_{1i}(\bt) \kappa_{1j}(\bt) \cos \angle (\bmu_{1i} \!-\! \bt, \bmu_{1j} \!-\! \bt)\\
	    &\hspace{4pt} + \kappa_{1i}^{2} (\bt) + \kappa_{1j}^{2} (\bt)
	    \end{aligned}
	} \label{eqn:K1i1j_lb_2}\\
	&\geqslant \sqrt{\kappa_{1i}^{2} (\bt) \!+\! \kappa_{1j}^{2} (\bt) \!+\! 2\kappa_{1i}(\bt) \kappa_{1j}(\bt) \cos A } \label{eqn:K1i1j_lb_3}\\
	&= \underbar{K}_{1i1j}(\bt) \label{eqn:K1i1j_lb_4}
	\end{align}
	where \eqnref{eqn:K1i1j_lb_3} is a consequence of the spherical triangle inequality and Lemma~\ref{lm:translation_uncertainty_angle}, since
	\begin{align}
	\angle (\ba, \bbb)
	&\leqslant \angle(\bc, \bd) + \angle(\ba, \bc) + \angle(\bbb, \bd) \label{eqn:K1i1j_lb_2_1}\\
	&\leqslant \angle (\bc, \bd) + \psi_{t}(\bmu_{1i}, \sC_t) + \psi_{t}(\bmu_{1j}, \sC_t) \label{eqn:K1i1j_lb_2_2}
	\end{align}
	for $\ba = \bmu_{1i} \!-\! \bt$, $\bbb = \bmu_{1j} \!-\! \bt$, $\bc = \bmu_{1i} \!-\! \bt_0$, and $\bd = \bmu_{1j} \!-\! \bt_0$.
		Also observe that $\forall(\br,\bt) \in (\sC_r \times \sC_t)$,
	\begin{align}
	\!K_{1i2j}(\bR_{\br},\bt) \!
		&=\!\! \sqrt{
	    \begin{aligned}
	    &\!2\kappa_{1i}(\bt) \kappa_{2j} \cos \! \angle (\bmu_{1i} \!-\! \bt, \bR_{\br}^{\!-1} \hat{\bmu}_{2j})\\
	    &\hspace{4pt} + \kappa_{1i}^{2} (\bt) + \kappa_{2j}^{2}
	    \end{aligned}
	} \label{eqn:K1i2j_ub_2}\\
    &\leqslant \sqrt{\kappa_{1i}^{2} (\bt) \!+\! \kappa_{2j}^{2} \!+\! 2\kappa_{1i}(\bt) \kappa_{2j} \cos B } \label{eqn:K1i2j_ub_3}\\
	&= \overbar{K}_{1i2j}(\bt) \label{eqn:K1i2j_ub_4}
	\end{align}
	where \eqnref{eqn:K1i2j_ub_3} follows from the reverse triangle inequality in spherical geometry and Lemmas~\ref{lm:translation_uncertainty_angle} and \ref{lm:rotation_uncertainty_angle}.
		With $\underbar{K}_{1i1j}$ and $\overbar{K}_{1i2j}$, a valid lower bound for \eqnref{eqn:objective_function} can be constructed by observing that $Z(x)$ \eqnref{eqn:Z} is a monotonically increasing function for $x \geqslant 0$
		and the dependency of $\kappa_{1i}$ on $\bt$ can be optimized separately.
		See the appendix for the full proof.
\end{proof}

\section{The GOSMA Algorithm}
\label{sec:algorithm}

The Globally-Optimal Spherical Mixture Alignment (GOSMA) algorithm is outlined in \algoref{alg:gosma}. It employs a depth-first search strategy using a priority queue (\lineref{alg:priority_queue}) where the priority is inverse to the lower bound. The algorithm terminates with $\epsilon$-optimality, whereby the difference between the best function value $d^*$ and the global lower bound $\underbar{d}$ is less than $\epsilon$ (\lineref{alg:stopping_criterion}).
Branching and bounding is performed on the GPU (\lineref{alg:kernel}), with each thread computing the bounds for a single branch.

\begin{algorithm}
\begin{algorithmic}[1]
\Require GMM--vMFMM pair, threshold $\epsilon$, domain $\Omega_r \!\times\! \Omega_t$
\Ensure optimal function value $d^*$, camera pose $(\br^*,\bt^*)$
\State $d^* \gets \infty$\label{alg:d_initialisation}
\State Add domain $\Omega_r \times \Omega_t$ to priority queue $Q$
\Loop
\State Update lowest lower bound $\underbar{d}$ from $Q$
\State Remove set of hypercubes $\{\sC_i\}$ from $Q$\label{alg:priority_queue}
\State \textbf{if} $d^* - \underbar{d} \leqslant \epsilon$ \textbf{then} terminate\label{alg:stopping_criterion}
\State Evaluate $\overbar{d}_{ij}$ \& $\underbar{d}_{ij}$ in parallel for subcubes of $\{\sC_i\}$\label{alg:kernel}
\ForAll{sub-cubes $\sC_{ij} \in \{\sC_i\}$}
\State \textbf{if} $\overbar{d}_{ij} \!<\! d^*$ \textbf{then} $(d^*, \br^*, \bt^*) \gets \text{SMA}(\br_{0ij}, \bt_{0ij})$\label{alg:sma}\label{alg:sma_criterion}
\State \textbf{if} $\underbar{d}_{ij} \!<\! d^*$ \textbf{then} add $\sC_{ij}$ to queue $Q$\label{alg:push_queue}
\EndFor
\EndLoop
\end{algorithmic}
\caption{GOSMA: a globally-optimal spherical mixture alignment algorithm for camera pose estimation}
\label{alg:gosma}
\end{algorithm}

We also developed a local optimization algorithm denoted as Spherical Mixture Alignment (SMA), which was integrated into GOSMA (\lineref{alg:sma}). We use the quasi-Newton L-BFGS algorithm \cite{byrd1995limited} to minimize \eqnref{eqn:objective_function}, with the gradient derived in the appendix. SMA is run whenever the BB algorithm finds a sub-cube $\sC_{ij}$ that has an upper bound less than the best-so-far function value $d^*$ (\lineref{alg:sma_criterion}), initialized with the center transformation of $\sC_{ij}$. In this way, BB and SMA collaborate, with SMA quickly converging to the closest local minimum and BB guiding the search into the convergence basins of better local minima. SMA accelerates convergence since reducing $d^*$ early allows larger branches to be culled (\lineref{alg:push_queue}), greatly reducing the search space.

\section{Results}
\label{sec:results}

The GOSMA algorithm, denoted GS, was evaluated with respect to the baseline algorithms SoftPOSIT \cite{david2004softposit}, BlindPnP \cite{moreno2008pose}, RANSAC \cite{fischler1981random}, and GOPAC \cite{campbell2018globally}, denoted SP, BP, RS and GP respectively, using both synthetic and real data.
For RANSAC, we use the P$3$P algorithm \cite{kneip2011novel}, the OpenGV framework \cite{kneip2014opengv}, and randomly-sampled correspondences.
To generate GMMs and vMFMMs, we cluster the point-set with DP-means \cite{kulis2012revisiting} and the bearing vector set with DP-vMF-means \cite{straub2015small}, and fit maximum likelihood mixture models to the clusters. These methods automatically select a parsimonious representation that adapts to the complexity of the scene geometry. 
We report the median translation error (in metres), rotation error (degrees), and runtime (seconds) including on-the-fly mixture generation. We also report the success rate, a summary statistic defined as the fraction of experiments where the correct pose was found: an angular error less than $0.1$ radians and a relative translation error less than 5\%.
Except where otherwise specified, the normalized $L_2$ distance threshold $\epsilon$ was set to $0.1$, the point-to-camera limit $\zeta$ was set to $0.5$, the scale parameters for mixture model generation $\lambda_{p}$ and $\lambda_{f}$ were set to $0.25$m and $2^{\circ}$ respectively, and semantic information was used in the real data experiments only, with class weights $\phi_{l} = |\Lambda|^{-1}$, the inverse of the number of classes.
All experiments were run on a 3.4GHz CPU and two GeForce GTX~1080Ti GPUs, and the C\texttt{++} code is available on the first author's website.

\subsection{Synthetic Data Experiments}
\label{sec:results_synthetic}

To evaluate GOSMA under a range of perturbations, 25 independent Monte Carlo simulations were performed per parameter setting using the framework of BlindPnP \cite{moreno2008pose}: $N_{I}$ random 3D point inliers and $\omega_{\text{3D}} N_{I}$ outliers were generated from $[-1,1]^3$; the inliers were projected to a $640\times480$ virtual image with a focal length of $800$; normal noise with $\sigma = 2$ pixels was added to the 2D points; and $\omega_{\text{2D}} N_{I}$ random outlier points were added to the image. An example of the data and alignment results is shown in \figref{fig:results_3d2dbounds}.

The time evolution of the global upper and lower bounds is shown in \figref{fig:results_bounds}. The plot reveals how local and global optimization strategies collaborate to decrease the upper bound with BB guiding the search into better convergence basins and SMA jumping to the nearest local minimum (the staircase pattern). It also shows that the majority of the runtime is spent increasing the lower bound, indicating that GOSMA will often find the global optimum when terminated early, albeit without an optimality guarantee.

\begin{figure}[!t]\centering
\begin{subfigure}[]{0.32\columnwidth}\centering
    \includegraphics[trim=0pt -18pt 0pt -18pt, clip, width=\columnwidth]{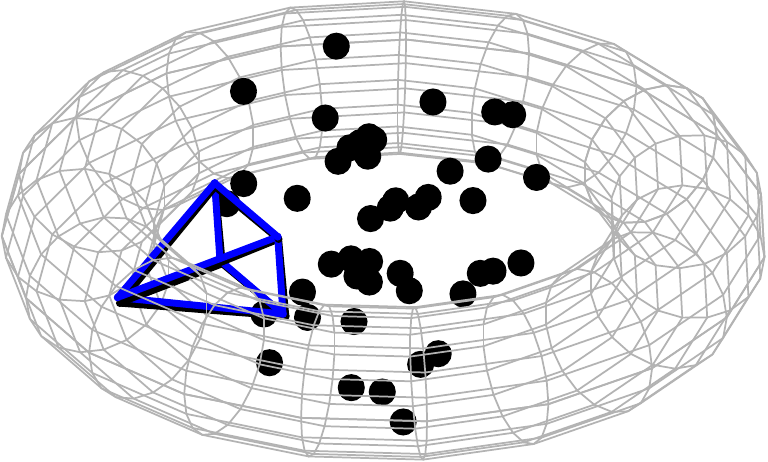}
    \subcaption{3D Results}
    \label{fig:results_3d}
\end{subfigure}\hfill
\begin{subfigure}[]{0.32\columnwidth}\centering
    \includegraphics[trim=0pt 0pt 0pt 0pt, clip, width=\columnwidth]{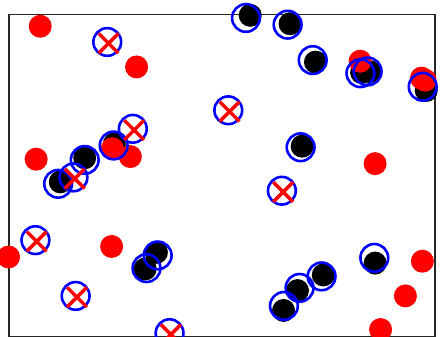}
    \subcaption{2D Results}
    \label{fig:results_2d}
\end{subfigure}\hfill
\begin{subfigure}[]{0.32\columnwidth}\centering
    \includegraphics[trim=0pt 0pt 0pt 0pt, clip, width=\columnwidth]{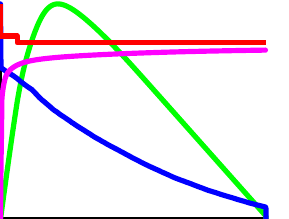}
    \subcaption{Bound Evolution}
    \label{fig:results_bounds}
\end{subfigure}
\caption{Sample 2D and 3D results for the random point data. \protect\subref{fig:results_3d}~3D points, true and GOSMA-estimated camera fulcra (completely overlapping) and toroidal pose prior. \protect\subref{fig:results_2d}~2D points (dots) and 3D points projected using the GOSMA-estimated camera pose (circles), with 2D and 3D outliers shown in red. \protect\subref{fig:results_bounds}~Evolution over time of the upper (red) and lower (magenta) bounds, remaining unexplored volume (blue) and queue size (green) as a fraction of their maximum values.}
\label{fig:results_3d2dbounds}
\end{figure}

To facilitate fair comparison with the local methods SoftPOSIT and BlindPnP, a torus pose prior was used for these experiments. It constrains the camera center to a torus around the 3D point-set with the optical axis directed towards the model \cite{moreno2008pose}. The torus prior was represented as a 50 component GMM for BlindPnP and 50 initial poses for SoftPOSIT. GOSMA and GOPAC were given a set of translation cubes that approximated the torus and were not given any rotation prior. RANSAC was set to explore correspondence space for up to 120s.
The results are shown in \figref{fig:results_synthetic}. Runtime values are clipped to an upper limit of 120s so that the scale is interpretable. GOSMA and GOPAC outperform the other methods, reliably finding the correct pose while still being relatively efficient.
While GOSMA has longer runtimes in the first two experiments, it has much better behaviour than the other methods when 2D outliers are present. For example, when $\omega_{\text{2D}}=1$, the median runtime of GOPAC (167s) is more than 30x higher than GOSMA (5s), while both always find the correct pose, with median angular errors below $1^{\circ}$ and relative translation errors below 2\%.
In fact, this random point setup significantly favors point-based methods such as GOPAC at the expense of our approach. For real surfaces, GOSMA is able to leverage its ability to adaptively compress the data, allowing it to quickly process a very large number of points.

\begin{figure}[!t]\centering
\hspace{-3pt}
\includegraphics[width=\columnwidth]{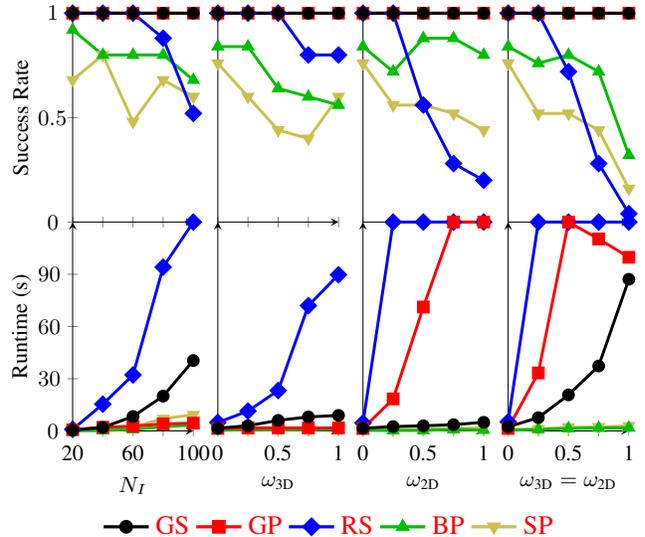}
\ref{plotlegend}
\caption{Results for the random points dataset with the torus prior. The success rates and median runtimes are plotted with respect to the number of inlier points ($N_{I}$), the fraction of additional 3D outliers ($\omega_{\text{3D}}$), 2D outliers ($\omega_{\text{2D}}$), and both, with default parameters $N_{I} = 30$ inlier points and $\omega_{\text{3D}} = \omega_{\text{2D}} = 0$, for 25 Monte Carlo simulations per parameter value.}
\label{fig:results_synthetic}
\end{figure}

\subsection{Real Data Experiments}
\label{sec:results_real}

The Stanford 2D-3D-Semantics (2D-3D-S) \cite{armeni2017joint} dataset contains panoramic images, point-sets, and semantic annotations for both modalities. It is a large indoor dataset with
approximately 1~million points per room and 8~million pixels per photo, collected using a structured-light RGBD camera.
We evaluated our algorithm on area~3 of the dataset, which contains lounges, offices and a conference room. The test data has 33 panoramic images taken from distinct camera poses where the camera is at least 50cm from any surface, and covers 13 rooms. Each room is a separate point-set, which models visibility constraints but assumes that the camera's position is known to the room level. Using this information, we set the translation domain to be the room size.
Semantic information is used by all methods in these tests: GOPAC and RANSAC use the pre-processing strategy from Campbell \etal~\cite{campbell2018globally}, selecting points and pixels from furniture classes only, whereas GOSMA uses class labels during optimization \eqnref{eqn:optimization_semantic}, making more effective use of the information. We also randomly downsample the point-sets and images to 100k points and pixels to reduce the mixture generation time.
The mixture scale parameters $\lambda_{p}$ and $\lambda_{f}$ \cite{kulis2012revisiting, straub2015small} were automatically selected to yield approximately 10 components per semantic class, 60--100 components in total.
For GOPAC, the inlier threshold $\theta$ was set to $2.5^{\circ}$ and the angular tolerance $\eta$ was set to $0.25^{\circ}$.

Qualitative and quantitative results are given in \figref{fig:results_real_indoor_office} and \tabref{tab:results_real_indoor}. Note that GOPAC and RANSAC were terminated at 900s and 120s respectively. GOSMA outperforms the other methods considerably, finding the correct pose in all cases with a small median runtime.
We also tested GOSMA without semantic labels during optimization, only during pre-processing (GS\textsuperscript{-$\Lambda$}), the same as for GOPAC and RANSAC. While this is more accurate and much faster than GOPAC, optimizing across the semantic classes provides another large accuracy and runtime gain.
We would like to emphasize the difficulty of this problem setup: the algorithm is given an image, a point-set and semantic class labels, and is required to estimate the camera pose. Compared to the synthetic data experiments, the sheer number of points and pixels, many of which are outliers, precludes the use of traditional methods.

\begin{figure}[!t]\centering
	\begin{subfigure}[]{\columnwidth}\centering
		\includegraphics[trim=110pt 105pt 60pt 95pt, clip=true, width=\columnwidth, ]{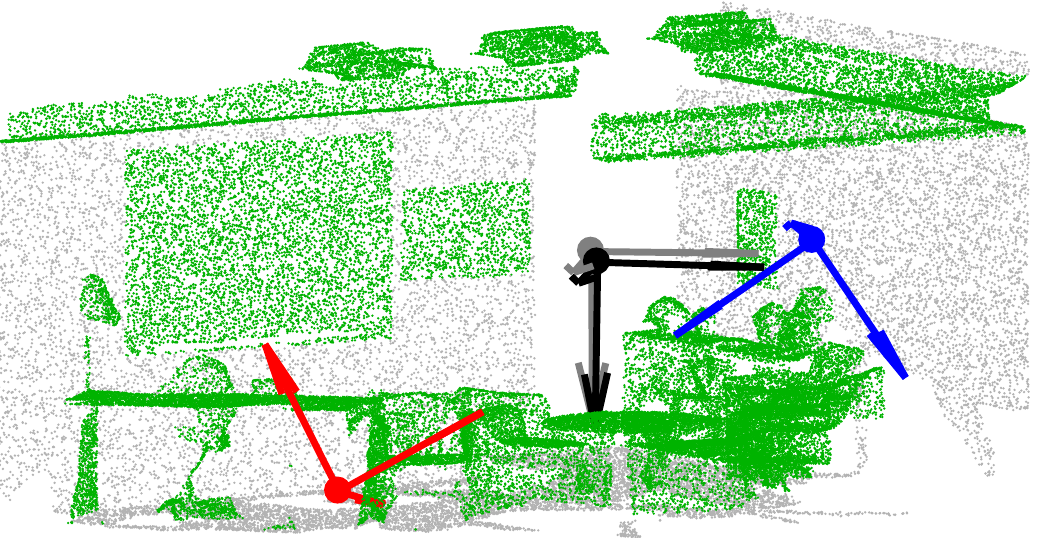}
		\subcaption{3D point-set and ground-truth (gray), GOSMA (black), GOPAC (red) and RANSAC (blue) camera poses. Object points are highlighted in green.}
		\label{fig:results_real_indoor_office_3d}	\end{subfigure}\vfill	\begin{subfigure}[]{\columnwidth}\centering
				\includegraphics[trim=70pt 130pt 70pt 110pt, clip, width=\columnwidth]{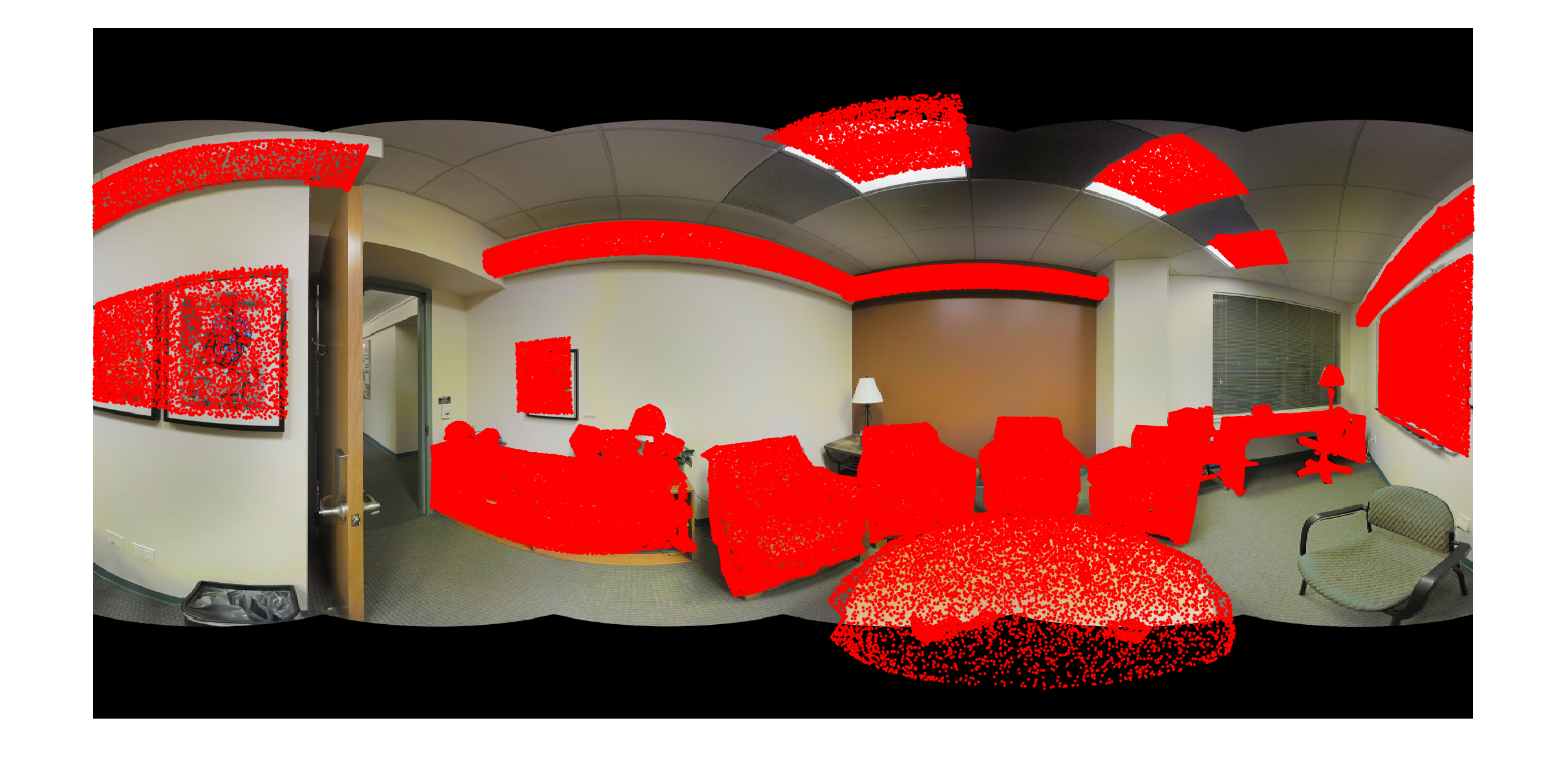}\vfill\vspace{3pt}		\includegraphics[trim=70pt 130pt 70pt 110pt, clip, width=\columnwidth]{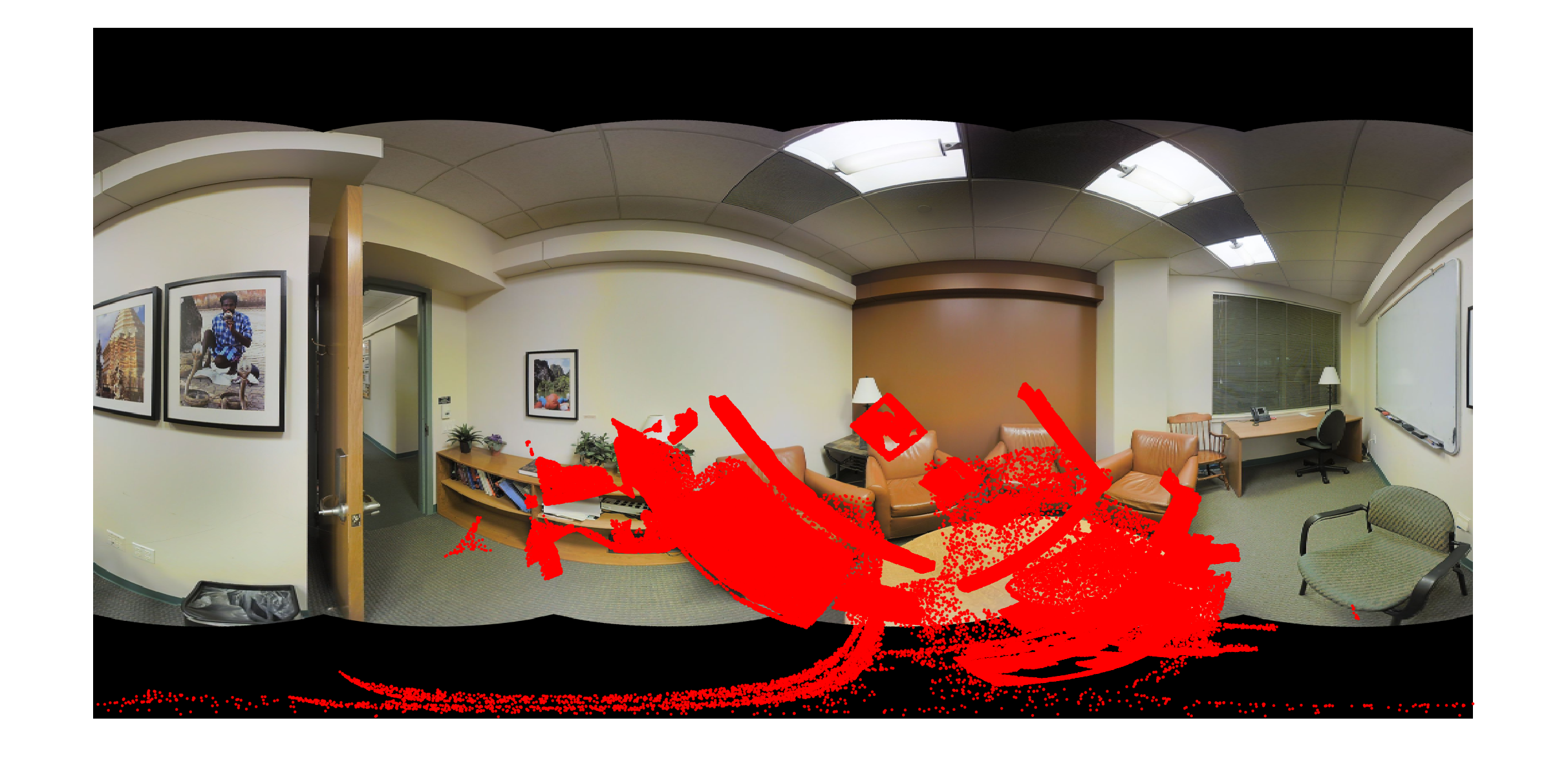}\vfill\vspace{3pt}		\includegraphics[trim=70pt 130pt 70pt 110pt, clip, width=\columnwidth]{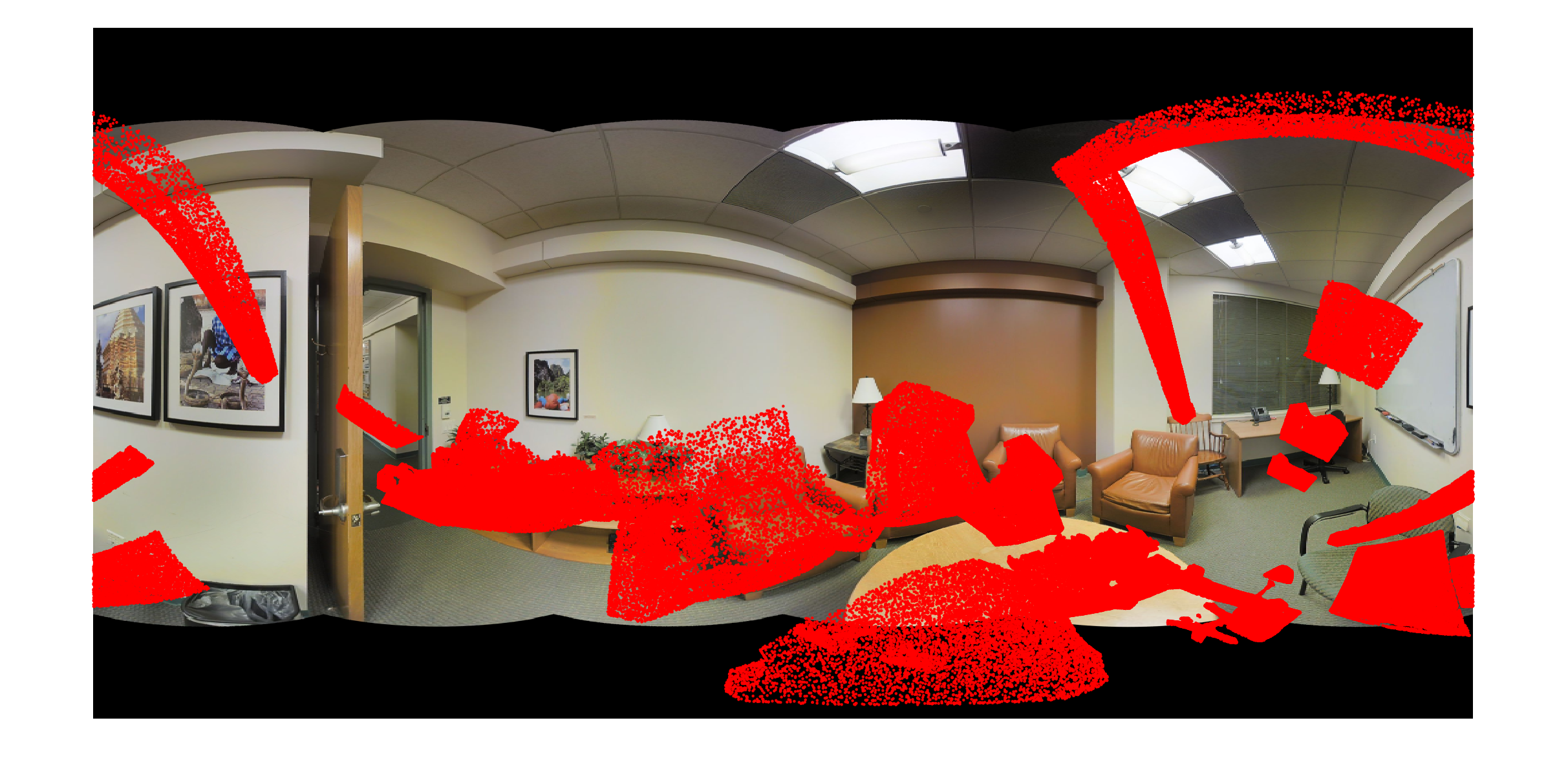}		\subcaption{3D points projected onto image using the GOSMA (top), GOPAC (middle), and RANSAC (bottom) camera poses. For clarity, only object points are plotted.}
		\label{fig:results_real_indoor_office_2d}	\end{subfigure}
	\caption{Qualitative camera pose results for office~3 of the Stanford 2D-3D-S dataset, showing the pose of the camera when capturing the image and the projection of 3D object points onto it. Only GOSMA found the correct pose as defined in this section. Best viewed in color.}
	\label{fig:results_real_indoor_office}
	\vspace{-6pt}
\end{figure}

\vspace{-5pt}
\setcounter{table}{1}
\begin{table}[!t]\centering
\caption{Camera pose results for GOSMA (GS), GOSMA without class labels during optimization (GS\textsuperscript{-$\Lambda$}), GOPAC (GP) and RANSAC (RS) for area~3 of the Stanford 2D-3D-S panoramic image dataset. Translation error, rotation error and runtime quartiles (\quartiles{\raisebox{-1pt}{$Q_1$}}{$Q_2$}{$Q_3$}) and the success rate are reported.
}
\label{tab:results_real_indoor}
\newcolumntype{C}{>{\centering\arraybackslash}X}
\renewcommand*{\arraystretch}{1.3}
\setlength{\tabcolsep}{2pt}
\begin{tabularx}{\columnwidth}{@{}l C C C C@{}}\hline
Method & GS & GS\textsuperscript{-$\Lambda$} & GP & RS\\\hline
Translation error (m) & \quartiles{0.05}{\textbf{0.08}}{0.15} & \quartiles{0.09}{0.14}{0.23} & \quartiles{0.10}{0.15}{0.27} & \quartiles{0.39}{0.56}{2.06}\\
Rotation error ($^{\circ}$) & \quartiles{0.91}{\textbf{1.13}}{2.18} & \quartiles{1.25}{2.38}{4.61} & \quartiles{2.47}{3.78}{5.10} & \quartiles{8.94}{18.3}{108}\\
Runtime (s) & \quartiles{1.4}{\textbf{1.8}}{4.4} & \quartiles{12.8}{19.1}{43.7} & \quartiles{448}{902}{902} & \quartiles{120}{120}{120}\\\hline
Success rate & \textbf{1.00} & 0.85 & 0.82 & 0.09\\\hline
\end{tabularx}
\end{table}

\section{Discussion and Conclusion}
\label{sec:conclusion}

In this paper, we have proposed a novel mixture alignment formulation for the camera pose estimation problem using the robust $L_2$ density distance on the sphere. Furthermore, we have developed a novel algorithm to minimize this distance using branch-and-bound, guaranteeing optimality regardless of initialisation.
To accelerate convergence, a local optimization algorithm was developed and integrated, GPU bound computations were implemented, and a principled way to incorporate side information such as semantic labels was devised.
The algorithm found the global optimum reliably on challenging datasets, outperforming other local and global methods.

This approach has several limitations, however.
Firstly, it scales quadratically with the number of mixture components, which scales with surface complexity.
Secondly, it is unable to resolve certain degenerate poses, such as when a wall fills the field-of-view of the camera. In this case, many camera poses satisfy the 2D information.
Thirdly, it does not use a geometric objective function, which reduces its interpretability.
A robust objective function in the image space such as intersection-over-union would be preferred, although it is not tractable for mixtures on the sphere.
Finally, the quality of its pose estimate depends on how well the mixtures represent the physical and projected surfaces in the real scene and image. While they can represent these surfaces arbitrarily accurately, the number of components is limited by practical considerations. Anisotropic densities would be preferred for this reason, however only isotropic densities, which model surfaces as points with a symmetric field of influence, have a tractable closed form on the sphere.
Hence, further investigation is warranted into aligning representations that model surfaces with fewer parameters, such as wireframes or meshes.

{\small
\bibliographystyle{ieee}

}

\end{document}